\documentclass{article}

\usepackage{arxiv}

\usepackage[utf8]{inputenc} 
\usepackage[T1]{fontenc}    
\usepackage{hyperref}       
\usepackage{url}            
\usepackage{booktabs}       
\usepackage{amsfonts}       
\usepackage{nicefrac}       
\usepackage{microtype}      
\usepackage{lipsum}
\usepackage{graphicx}
\graphicspath{ {./images/} }

\usepackage{amsmath,amsthm,bbm}
\usepackage{amsfonts,amssymb}
\usepackage{amsmath,amsxtra,amssymb,latexsym, amscd,amsthm}
\usepackage{indentfirst}
\usepackage{fancyhdr}
\usepackage{subfigure}
\usepackage{graphicx}
\usepackage[mathscr]{eucal}
\usepackage{mathrsfs}
\usepackage{makeidx}
\usepackage{mathabx}
\usepackage{multicol}
\usepackage{color}
\usepackage{multirow}
\usepackage{bm}
\usepackage{fancyhdr}
\usepackage[switch,columnwise]{lineno}
\usepackage{lipsum}
\usepackage{fancyhdr}

\newcommand{\NN}{\mathbb{N}}
\newcommand{\RR}{\mathbb{R}}

\newtheorem{thm}{Theorem}
\newtheorem{cor}[thm]{Corollary}
\newtheorem{lem}[thm]{Lemma}

\newtheorem{prop}[thm]{Proposition}
\newtheorem{remark}[thm]{Remark}
\newtheorem{defn}{Definition}

\newcommand{\thmref}[1]{Theorem~\ref{#1}}
\newcommand{\lemref}[1]{Lemma~\ref{#1}}
\newcommand{\corref}[1]{Corollary~\ref{#1}}
\newcommand{\propref}[1]{Proposition~\ref{#1}}

\def\f{\frac}

\def\s{\sigma}
\def\d{\delta}

\def\D{\mathcal{D}}
\def\CB{\mathcal{B}}
\def\C{\mathcal{C}}

\def\P{\mathcal{P}}

\def\S{\mathcal{S}}
\def\T{\mathcal{T}}

\def\var{\varepsilon}

\def\RR{\mathbb{R}}

\def\Diag{\operatorname{Diag}}
\def\Span{\operatorname{Span}}

\newcommand{\setsep}{\;:\;}

\def\la{\langle}
\def\ra{\rangle}

\def\rank{\operatorname{rank}}
\title{Approximation analysis of CNNs from a feature extraction view}

\author{
	{\bf Jianfei Li}\thanks{Department of Mathematics, City University of Hong Kong
		(\texttt{jianfeili2-c@my.cityu.edu.hk})}
	\and
	 {\bf Han Feng}\thanks{Department of Mathematics, City University of Hong Kong (\texttt{hanfeng@cityu.edu.hk})}
	\and
	{\bf Ding-Xuan Zhou}\thanks{School of Mathematics and Statistics, The University of Sydney
		(\texttt{dingxuan.zhou@sydney.edu.au})}
}

\date{}

\begin{document}
	\maketitle

\begin{abstract}
Deep learning based on deep neural networks has been very successful in many practical
applications, but it lacks enough theoretical understanding due to the network architectures
and structures. In this paper we establish some analysis for linear feature extraction by a deep multi-channel convolutional neural networks (CNNs), which demonstrates the power of deep learning over traditional linear transformations, like Fourier, wavelets, redundant dictionary coding methods. Moreover, we give an exact construction presenting how linear features extraction can be conducted efficiently with multi-channel CNNs. It can be applied to lower the essential dimension for approximating a high dimensional function. Rates of function approximation by such deep networks implemented with channels and followed by fully-connected layers are
investigated as well. Harmonic analysis for factorizing linear features into multi-resolution convolutions plays an essential role in our work. Nevertheless, a dedicate vectorization of matrices is constructed, which bridges 1D CNN and 2D CNN and allows us to have corresponding 2D analysis.
\end{abstract}
\section{Introduction}
Deep learning has been a powerful tool in processing big data from various practical
domains \cite{LeCun1998, Krizhevsky2012, Goodfellow2016, Hinton}. Its power is mainly brought by deep neural networks with special
structures and network architectures which are designed to capture hierarchical data features
efficiently. Deep convolutional neural networks (DCNNs) form an important
family of structured deep neural networks. They are especially efficient in processing natural
speeches and images for speech recognition and image classification, and the involved
convolutional structures are believed to capture local shift-invariance properties of speech
and image data.

The classical fully-connected (multi-layer) neural networks do not involve special network structures. For processing data $x=\left(x_{i}\right)_{i=1}^{d} \in \mathbb{R}^{d}$, such a network $\Psi$ with $J$ hidden layers $\left\{H_{j}: \mathbb{R}^{d} \rightarrow \mathbb{R}^{d_{j}}\right\}_{j=0}^{J}$ with widths $\left\{d_{j}\right\}$ is defined with the input layer $H_{0}(x)=x$ of width $d_{0}=d$ by
\begin{align}\label{fully}
H_{j}(x)&=\sigma\left(F^{(j)} H_{j-1}(x)-b^{(j)}\right), \quad j=1,2, \ldots, J,\\
\Psi(x) &= c^T H_{J}(x)
\end{align}
with an activation function $\sigma: \mathbb{R} \rightarrow \mathbb{R}$ acting componentwise on vectors, connection matrices $F^{(j)} \in \mathbb{R}^{d_{j} \times d_{j-1}}$, weight vector $c \in \RR^{d_J}$, and bias vectors $b^{(j)} \in \mathbb{R}^{d_{j}}$.
These networks have been well understood due to a large literature on function approximation around 30 years ago \cite{Leshno1993, Mhaskar1993}, and the recent work on training parameters by backpropagation, stochastic gradient descent, and error-correction tuning \cite{Goodfellow2016}. The nice approximation properties of the classical neural networks are mainly caused by the fully-connected nature of \eqref{fully} where $F^{(j)} \in \mathbb{R}^{d_{j} \times d_{j-1}}$ is a full matrix consisting of $d_{j} d_{j-1}$ free parameters. This number of parameters to be trained is huge, even for shallow networks with one hidden layer corresponding to $J=1$, when the input data dimension $d$ is large and $d_1\gg d$. For example, in image processing, one starts with a digital image $X:\{0,1, \ldots, d\}^{2} \rightarrow \mathbb{R}$ with the side width $d$ in the order of hundreds. Applying vectorization to $X$ leads to the input data of size $(d+1)^{2}$. Then the number of parameters to be trained in \eqref{fully}  would be at least in the order of tens of billions, which is too huge to be implemented in practice.

  Let $s,t,d$ be integers. Here and below we set $\D(d,t):=\lceil d/t\rceil$ the smallest integer no less than $d/t$ and $U_{i,j}$ the $(i,j)$th entry of a matrix $U\in \RR^{m\times n}$ with $U_{i,j}=0$ if $i\notin [1,m]$ or $j\notin [1,n]$. Similar setting works for vectors as well. For simplicity, without specification all convolutions we concern have no zero padding.

   The 1-D convolution of size $s$, stride $t$ is computed for a kernel $w\in \RR^s$ and  a digital signal $x\in \RR^d$ to be $ w*_t x\in \RR^{\D(d,t)}$ by
  \begin{equation}\label{def-convolution}
  (w *_t x )_{i}=\sum_{k=1}^{s} w_{k} x_{(i-1)t+k}, \ i=1,\ldots, \D(d,t).
  \end{equation}
  It is common for  filter sizes to be small, such as $2, 3, 5$, or even $1$.
   For example, let $x\in \RR^{6}$ and $w\in \RR^3$,
   \[w*_2x =(x_1w_1+x_2w_2+x_3w_3, x_3w_1+x_4w_2+x_5w_3, x_5w_1+x_6w_2)^T\in \RR^3.\]
   For any stride $t$ and input dimension $d$,  if we define the Toeplitz operator of a filter $w$ by
   $$\T_{t,d}(w)=\left[w_{j-(i-1)t}\right]_{i=1, \ldots, \D(d,t), j=1, \ldots, d} \in \mathbb{R}^{\D(d,t) \times d},
  $$
      the 1-D convolution with stride $t$ for a signal $x\in \RR^d$ can be equivalently conducted by the multiplication of a Toeplitz type $\D(d,t)\times d$  matrix
\begin{equation}\label{Toeplitz}
\T_{t,d}(w)
 \in \mathbb{R}^{\D(d,t)\times d}
\end{equation}
with $x$; that is,
 $$w*_tx =\T_{t,d}(w)x.$$
 Let
 	$$\Diag_k(A)=\left.\left(
 	\begin{array}{cccc}
 	A & 0 & \cdots & 0 \\
 	0 & A & \ddots& \vdots \\
 	\vdots & \ddots & \ddots & 0 \\
 	0 & 0 & 0 & A \\
 	\end{array}
 	\right)\right\}k
 	$$
 	for any $A\in \RR^{m\times n}$. Then, particularly $\Diag_k(w^T)=\T_{t, t^k}(w)\in \RR^{t\times tk}$ for $w\in \RR^t$.

Furthermore, a multi-channel DCNN consists of more than 1 convolutional layers and multiple channels in each layer which is defined as follows.
\begin{defn}\label{1dcnn}
Given the input size $d \in \mathbb{N}$, a multi-channel $D C N N$ in 1$D$ of depth $J$ with respect to filter size $\{s_j\geq 2\}_{j=1}^J$, stride $\{t_j\geq 1\}_{j=1}^J$ and channel $\{n_j\geq 1\}_{j=1}^J$,
 is a neural network $\left\{h_{j}: \mathbb{R}^{d} \rightarrow \mathbb{R}^{d_j\times n_j}\right\}_{j=1}^{J}$ defined iteratively for which  $h_{0}(X)=X \in \mathbb{R}^{d}$, $d_0=d$, $n_0=1$, and the $\ell$th channel of $h_{j}$ is given by
$$
h_{j}(X)_{\ell}=\sigma\left(\sum_{i=1}^{n_{j-1}} W^{(j)}_{\ell,i}*_{t_j} h_{j-1}(X)_i+B^{(j)}_\ell\right),$$
for $j=1, \ldots, J, \ \ell=1,\ldots, n_{j}$, where $B^{(j)}_\ell \in \mathbb{R}$ are biases, $W^{(j)}_{\ell,i}\in \RR^{s_j}$ are filters, $d_j=\D(d_{j-1},t_j)$, and $\sigma(u)=\max \{u, 0\}$, $ u \in \mathbb{R}$ is the rectified linear unit (ReLU) activation.
\end{defn}
Turning a Convolutional Neural Network (CNN) into a linear transformation can be achieved by adjusting the bias terms to a sufficiently high value.


An approximation theory for the 1-D DCNN with single channel induced by convolutional matrices like \eqref{Toeplitz} have been extensively studied in a series of papers by Zhou et al. \cite{Zhou2020_1,Zhou2020_2, FFHZ,ZhouMao, ZhouMaoShi}. Afterwards,  Xu et al.\cite{he2022approximation} extend to approximate $L_2$ functions by 2D CNNs with respect to odd filter size and stride one. It requires intermediate layers in the same dimension as the input through convolutions, which is slightly compatible with the practice. Specifically, Bolcskei et al. \cite{bolcskei2019optimal} showed that affine systems can be approximated by deep sparse connected neural networks.

%

The vast majority of DCNNs  in most practical applications of deep learning focus on image processing and are induced by 2-D convolutions. As an analogy of the $1$-D case, 2-D convolutions and DCNNs can be well defined. Precisely, here the 2-D convolution of a filter $W\in \RR^{s\times s}$ and a digital image $X\in \RR^{d\times d}$ with stride $t$  produce a convoluted matrix $W \circledast _t X\in \RR^{\D(d,t)\times \D(d, t) }$ for which the $(i,j)$-entry is given by
$$
(W \circledast_t X)_{(i,j)}:=\sum_{\ell_1,\ell_2=1}^s W_{\ell_1,\ell_2} X_{(i-1)t+\ell_1, (j-1)t+\ell_2}, \quad i,j=1,\ldots, \D(d,t).
$$

\begin{defn}\label{2dcnn}
Given input signal $X\in \RR^{d\times d}$, a multi-channel $D C N N$ in 2$D$ of depth $J$ with respect to filter size $s_j\geq 2$, stride $t_j\geq 1$ and channel $n_j\geq 1$, $j=1,\ldots, J$
 is a neural network $\left\{h_{j}: \mathbb{R}^{d\times d} \rightarrow \mathbb{R}^{d_j\times d_j\times n_j}\right\}_{j=1}^{J}$ defined iteratively for which  $h_{0}(X)=X \in \mathbb{R}^{d\times d}$, $d_0=d$, $n_0=1$, and the $\ell$th channel of $h_{j}$ is given by
$$
h_{j}(X)_{\ell}=\sigma\left(\sum_{i=1}^{n_{j-1}}W_{\ell,i}^{(j)}\circledast_{t_j} h_{j-1}(X)_i+B_\ell^{(j)}\right), $$
 for $j=1, \ldots, J,\ \ell=1,\ldots, n_{j}$, where $W_{\ell,i}^{(j)}\in \RR^{s_j\times s_j}$ are filters and $B^{(j)}_\ell \in \mathbb{R}$ are bias, and $d_j=\D(d_{j-1},t_j)$.
\end{defn}
Here the restriction of input channel to be $1$ is actually not necessary.
%


In this paper,
we will focus on the approximation by deep convolutional neural networks to functions with low-dimensional structures.
In fact, numerous data that arise from biology, commercial, and financial
activities or from social networks do exhibit very good patterns that can be well
approximated by low-dimensional models, as people can see from plenty of examples\cite{tenenbaum2000global,roweis2000nonlinear,moon2019visualizing,wright2022high,lee2020stock,geng2015learning}.
In classical signal processing, the intrinsic low-dimensionality of data is mostly
exploited for purposes of efficient sampling, storage, and transport \cite{OSB99, PV08}. In applications
such as communication, it is often reasonable to assume the signals of
interest mainly consist of limited frequency component or sparse coding under a redundant basis.
%
Precisely, we take hypotheses that input
data, say $x\in \RR^d $(or $\RR^{d\times d}$), are represented in terms of linear combination of a set of
elementary patterns (or features) $v^{(\ell)}\in \RR^d$(or $\RR^{d\times d}$):
\[x=\sum_{\ell=1}^m a_\ell v^{(\ell)}+\var\]
where  $a=(a_1,\ldots, a_m)\in \RR^m$ are sparse coefficients and $\var\in \RR^d$ is some small
modeling error. The collection of all patterns $ v^{(\ell)}$ is called
a dictionary. We shall study the crucial problem how a dictionary and sparse representation can be learned from the deep CNNs.

The contribution of this work is threefold. First, we develop the theory of linear feature extraction of CNNs. Notice that there are plenty of experiments observing that in deep CNNs, convolutional kernels near the input layer pick up detailed textures of images, and those near the output layer extract abstract information. Our theory can tell what kind of features can be extracted and how receptive fields perform when layers go deeper. Second, we characterize the expressive capacity of deep classifiers (a CNN for feature extraction and a fully connected network for prediction). The functions with smoothness or defined on a low-dimensional manifold are considered as the target function space. Finally, for data from separable affine spaces, we construct a deep classifier that can approximate the labels well.

\section{Linear feature extraction by DCNNs}

In this section, we shall establish the following theorems which illustrate the power of DCNN for extracting linear features.

\subsection{The case of 1D-CNN with filter size 2}
For convenience, our discussion will first focus on convolutions with constant filter size $2$ and stride $2$ and then extend to general cases in sequent subsections. The following notation is necessary for our statement and proof. For any vector $v\in \RR^d$ and $s\leq d$, we define the patches collection of $v$ of length $s$ by
$$\P_s(v) = \left\{ (v_{js+1}, \dots, v_{s(j+1)})\in \RR^s: j = 0, \dots, n-1, n=\left\lceil \frac d s\right\rceil\right\} $$
where $v_\ell=0$ if $\ell>d$.
For a set $A$ of vectors, we denote by $\S(A)$ a finite set such that $\Span\left(\S(A)\right) = \Span(A)$.
\begin{thm}\label{1D-extractation}
  Let $J, m\geq 1$ be integers and $\Omega$ be a compact set of $\RR^{2^J}$. For any dictionary of row vectors $\{v^{(\ell)}\}_{\ell=1}^m\subset \RR^{2^J}$, there exists a multi-channel DCNN $\{h_j\}_{j=1}^J$ of depth $J$ with  constant filter size $2$ and stride $2$ such that $h_J$ has $m$ channels and
  \begin{equation}\label{inner-product}
  \langle x, v^{(\ell)}\rangle= h_J(x)_\ell, \quad \forall \ x\in \Omega \subset \RR^{2^J}, \ \ell=1,\ldots, m.
  \end{equation}

  In addition, for any $j=1,\ldots, J-1$,
  \[ \left \{w *_2 x: w \in \S\left(\cup_{\ell=1}^m\P_{2^j}( v^{(\ell)})\right)  \right \} \subset \{h_j(x)_{\ell}: \ell=1,\ldots, n_j\}.\]
\end{thm}
%
%
 We postpone proof and construction details to Appendix A.  One can find how the intermediate convolutional layers can give the coherence between input and multi-scale patches of dictionary, which somehow verifies a widely recognized knowledge \cite{DNNbook} that {\it``the next convolution layer will be able to put together many
of these patches together to create a feature from an area of the image that is larger. The
primitive features learned in earlier layers are put together in a semantically coherent way to
learn increasingly complex and interpretable visual features.''.    } \thmref{1D-extractation} is compatible with the experimental observation that as layers go deeper, receptive fields become larger. Another important characteristic is that the intermediate features are associated with a certain dictionary. The flexibility of choice of kernels makes CNNs strong for various learning tasks when optimization algorithms are powerful.
  Moreover, for stride-1 convolutions, from the feature extraction view it can be interpreted to enrich to contain shifted features $\la \operatorname{shift}_1(x), \bar v^{(\ell)}\ra$, where $\operatorname{shift}_1(x)=(x_2,\ldots, x_d, 0)$ is the left shift of $x$ by step $1$. It somehow explains some shift-invariance properties under CNNs.



%
%


%

\subsection{The case of 1D-CNN with mixed kernel sizes}
Next, we shall generalize our framework for convolutions with mixed filter sizes. The proof is analogous to above.
%
%

\begin{thm}\label{flexible kernel}
  Let $m,J, d$ be integers and $\Omega  $ be a compact set of $\RR^d$. Given any integers $s_1,\ldots, s_J$ such that $d\leq \prod_{j=1}^J s_j$, for any row vectors $v^{(\ell)}\in \RR^d$, $\ell=1,\ldots, m$, there exists convolutional layer $h_j$ with filter size $s_j$, stride $s_j$, $j=1,\ldots, J$, such that $h_J$ has $m$ channels and
  \[ \langle x,  v^{(\ell)}\rangle= (h_J(x))_\ell, \quad \forall \ x\in \Omega \subset \RR^{d}, \ \ell=1,\ldots, m.\]
\end{thm}

\begin{proof}
We give a constructive proof to verify the existence. Since one can always use large bias $B^{(j)}_{\ell}$ to push the convolution part in each convolutional layer above zero, without loss of generality we assume that at each layer convolution parts are positive and thus ReLU activation is omitted throughout the proof.

For the case when $d= \prod_{j=1}^J s_j$, it is readily  an extension of \thmref{1D-extractation} and thus \thmref{flexible kernel} . If $d<\prod_{j=1}^J s_j:=\tilde d$, we take
\[\tilde x=\C(x, \mathbf{0}_{\tilde d-d})\]
and
\[\tilde v^{(\ell)}=\C(\bar v^{(\ell)}, \mathbf{0}_{\tilde d-d}),\]
where $\mathbf{0}_k=(0,\ldots, 0)^T\in \RR^k$ and $\C $ represent the concatenation of two column vectors vertically.  Notice that
\[\langle x,  v^{(\ell)}\rangle=\langle \tilde x, \tilde v^{(\ell)}\rangle, \quad \ell=1,\ldots, m.\]
For $\{\tilde v^{(\ell)}\}$, we take
\[\{\tilde r_{k,j}, j=1,\ldots, \tilde n_k\}:=\S\left(\bigcup_{\ell=1}^m\P_{m_k}(\tilde v^{(\ell)})\right)\]
where $m_k=s_1\cdots s_k$ and $\tilde{n}_k\leq \min \{\tilde d/m_k,m_k \}$ is the dimension of the $k$-level partition space. For convenience, we view $\tilde r_{k,j}$ as a row vector without the transpose. Then for each $\tilde r_{k+1,j}$, $k=1,\ldots,J-1, j=1,\ldots, \tilde n_{k+1}$, we have $\{u^{(t)}_{k+1,j}\in \RR^{m_k}\}_{t=1}^{s_{k+1}}$ such that
\[\left(u^{(1)}_{k+1,j}, \ldots, u^{(s_{k+1})}_{k+1,j}\right)=\tilde r_{k+1,j}.\]
Since $u^{(t)}_{k+1,j}\in \P_{m_k}(\tilde v^{(\ell)})$, there exist $w_{i,j}^{(k+1,t)}\in \RR$ such that
\[u^{(t)}_{k+1,j}=\sum_{i=1}^{n_k} w_{i,j}^{(k+1,t)} \tilde r_{k,i}\]
for all $t=1,\ldots, s_{k+1}$. Therefore,
 \[ \tilde r_{k+1,j}=\sum_{i=1}^{n_k}(w^{(k+1,1)}_{i,j},\ldots, w^{(k+1,s_{k+1})}_{i,j}) \left(
                                               \begin{array}{cccc}
                                                 \tilde r_{k,i} & 0 & 0& 0\\
                                                 0 & \tilde r_{k,i}& 0& 0\\
                                                 \vdots& \vdots& \ddots& \vdots\\
                                                 0& 0& 0& \tilde r_{k,i}
                                               \end{array}
                                             \right).\]
 Set $W^{(1)}_{1,j}=\tilde r_{1,j}\in \RR^{s_1}$, $j=1,\ldots, \tilde n_1$ and
 \[W^{(k)}_{i,j}=(w_{i,j}^{(k,1)}, \ldots, w^{(k,s_k)}_{i,j})\in \RR^{s_k},\]
 for $k=2,\ldots,J$, $i=1,\ldots, \tilde n_{k-1}$, $j=1,\ldots, \tilde n_k$.
 Then the implement of multi-channel convolutions with the above filters $\{W^{(k)}_{i,j}\}$ and strides $s_k$ will give us
  \[h_J(\tilde x)_\ell=\langle \tilde x, \tilde v^{(\ell)}\rangle, \ \ell=1,\ldots, m. \]
  On the other hand, the definition \eqref{def-convolution} implies that $h_J(x)=h_J(\tilde x)$.
\end{proof}

%
%
%
%
%

\subsection{The case of 2D CNNs}

In this section, we turn to establish similar analysis for 2D DCNN, which will be done by connecting it to the 1D case. To accomplish it,
 we need the following concept of hierarchical partitions of matrices.

Denote $[n]=\{1,2,\ldots, n\}$ for any integer $n$.
For the index set $\Lambda=[d]\times [d]$ and a fixed resolution $J$,
let $\{\mathcal{B}_j\}_{j=0}^J$ be a family of subsets of $2^\Lambda$. We call  $\{\mathcal{B}_j\}_{j=0}^J$
a \emph{hierarchical partition} of $\Lambda$ with resolution $J$ if the following two conditions are satisfied:
\begin{enumerate}
	\item[{a)}] Root property: $\mathcal{B}_0 = \{\Lambda\}$ and each $\mathcal{B}_j$
	is a disjoint partition of $\Lambda$.
	\item[{b)}] Nested property: for any $j=1,\ldots, J$ and any (child) set $A \in \mathcal{B}_j$, there exists a (parent) set $B \in \mathcal{B}_{j-1}$ such that $A \subseteq B$. In other word, partition $\mathcal{B}_j$ is a refinement of partition $\mathcal{B}_{j-1}$.
\end{enumerate}
%
  For convenience, given a sequence of integers $s_1,\ldots, s_J\geq 2$,
  we set $d=\prod_{j=1}^J s_j$.
   For every $j=1,\ldots,J$, we take the hierarchical  partition by refining each  element in $\CB_{j-1}$  into $s_j^2$ sub-blocks of equal size in $\CB_{j}$. Furthermore, we will label sets in $\CB_j$ in the order of left to right and then top to bottom. Precisely, we make
\[
\CB_j=\{R_{\vec v}\subseteq \Lambda \setsep {\vec v\in I_j}\},
\]
where $I_j:= [s^2_1]\times\cdots\times[s^2_j]$, such that $ R_{(\vec v,t)}\subseteq R_{\vec v}$ for all $\vec v\in I_{j-1}$, $t\in[s_j^2]$ and $j=1,\ldots, J$. In this way, we denote by $\lambda(i,j):=(t_1,\ldots, t_J)$ such that  $ R_{(t_1,\ldots, t_J)}=\{(i,j)\}$.
 Based on the hierarchical partitions,  we can label each index $(i,j)\in \Lambda$ by {\bf partition  vectors}:
$$\lambda(i,j):=(t_1,\ldots, t_J)$$
such that  $ R_{(t_1,\ldots, t_J)}=\{(i,j)\}$. Furthermore, we define $\delta: [d]\times [d]\to [d^2]$ by
\[\delta(i,j)=1+\sum_{k=1}^J [\lambda(i,j)_k-1]\tilde s_k^2,\]
and  a bijective mapping $T: \RR^{d\times d}\to \RR^{d^2}$ from a matrix to its vectorization by $Y_{i,j}\mapsto \tilde Y_{\d(i,j)}$,
where $\tilde s_j=d/(s_1s_2\cdots s_{j})$, $j=1,\ldots,J$. Refer to Figure~\ref{acts1} for an illustration of our notations.

 %
%

      	\begin{figure}[http]
        		\centering
        		\includegraphics[width=8cm]{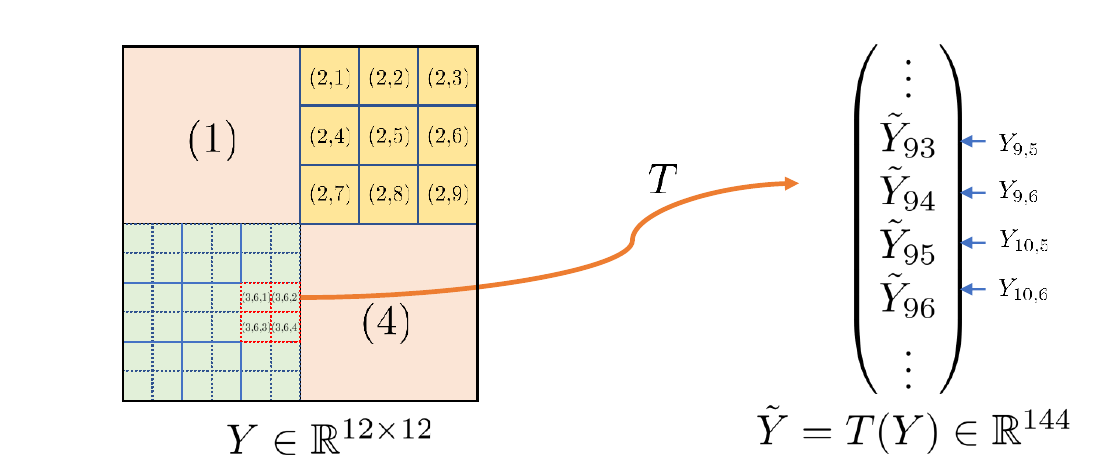}
        		\caption{Visualization of mapping $T$ for a kernel sequence $(2,3,2)$. Then $d=2\cdot 3\cdot 2=12$, $\tilde s_1=12/2=6$, $\tilde s_2=12/(2\cdot 3)=2$, $\tilde s_3=1$. For element $Y_{9,5}$,  $\lambda(9,5)=(3,6,1)$, $\d(9,5)=1+(3-1)6^2+(6-1)2^2+(1-1)1^2=93$, $Y_{9,5} \overset{T}\longmapsto \tilde Y_{93}$.}
        		\label{acts1}
        	\end{figure}
%

\begin{prop}\label{prop3}
Let $s_1,\ldots, s_J$ be integers and $d=\prod_{j=1}^J s_j$. For any given kernels $K_j\in \RR^{s_j\times s_j} $, $j=1,\ldots, J$,  and a 2D-signal $Y\in  \RR^{2^J\times 2^J}$,
  \begin{align}\label{2dconv}
  	\begin{aligned}
  		& \left(K_J\circledast_{s_J} \cdots\circledast_{s_3}  \left( K_2 \circledast_{s_2} \left( K_1 \circledast_{s_1}Y \right) \right) \right)  \\
  		=&  \left(   \tilde K_J*_{s_J^2}\cdots *_{s_2^2} \left( \tilde K_1*_{s^2_1} \tilde Y  \right) \right)
  	\end{aligned}
\end{align}
where $\tilde M$ is the reshaped vector of $M$ with respect to the above partition way.
\end{prop}

 It shows that taking 2D convolutions can be equivalent to corresponding 1D convolutions.

 \begin{proof} To be simple and without losing generality, we only consider the case for $s_1=s_2=\cdots=s_J=2$ and $d=2^J$.
 Then for any $2\times 2$ kernel $K=(K_{i,j})$, $\tilde K=(K_{1,1}, K_{1,2}, K_{2,1}, K_{2,2})^T$ and for $Y\in \RR^{d\times d}$, $\tilde Y\in \RR^{d^2}$ with $\tilde Y_{T(i,j)}=Y_{i, j}$ for $(i,j)\in [d]\times [d]$.

 It is obvious that \eqref{2dconv} holds for $J=1$. Suppose it holds for any $J\leq m$. If we divide $Y\in \RR^{2^{m+1}\times 2^{m+1}}$ into four blocks $Z_{1}=(Y_{i,j})_{(i,j\in R_{(1)})}$, $Z_{2}=(Y_{i,j})_{(i,j\in R_{(2)})}$, $Z_{3}=(Y_{i,j})_{(i,j\in R_{(3)})}$ and $Z_{4}=(Y_{i,j})_{(i,j\in R_{(4)})}$, then
 for each $t=1,\ldots, 4$, $Z_t\in \RR^{2^m\times 2^m}$ and thus
 \begin{equation}\label{2dconv_m}
 \left(  K_m\circledast_2\cdots\circledast_2 \left( K_1\circledast_2Z_t \right) \right)= \left(  \tilde K_m *_4\cdots *_4 \left( \tilde K_1*_4\tilde Z_t \right) \right)   \in \RR.
 \end{equation}
 On the other hand, by the definition of operator $T$, $\tilde Y=(\tilde Z_1^T,\tilde Z_2^T,\tilde Z_3^T,\tilde Z_4^T)^T$. It implies that
 \[ \left( \tilde K_m*_4\cdots*_4 \left( \tilde K_1*_4\tilde Y  \right) \right) =(W_1, W_2, W_3, W_4)^T\]
 where $W_t=  \left(  \tilde K_m*_4\cdots*_4  \left( \tilde K_1*_4\tilde Z_t  \right) \right)$. By \eqref{2dconv_m} and the fact that \[ \left( \tilde K_{m+1}*_4 \left( \tilde K_m *_4\cdots*_4 \left(\tilde K_1*_4\tilde Y \right)\right)\right)=\C(\tilde K_{m+1}*_4W_1  , \tilde K_{m+1}*_4W_2   , \tilde K_{m+1}*_4W_3  ,\tilde K_{m+1} *_4W_4   ),  \] we have \eqref{2dconv} holds for $J=m+1$ and therefore for any integer $J$.
 \end{proof}

In the following corollary, applying the analysis of 1D convolutions, we show that 2D convolutions can extract singular values of a matrix.
	\begin{cor}\label{cor4}
		Let $X\in \RR^{d\times d}$ be a matrix with $d=2^J$ and $r=\rank(X)$. If $\alpha_1\geq \alpha_2\geq \cdots \geq \alpha_r>0$ are nonzero singular values of $X$, then there exist convolutional layers of $\{h_j\}$ such that
		\[h_J(X)_\ell=\alpha_{\ell}, \  \ell=1,\ldots, r.\]
	\end{cor}
	
	\begin{proof}
		By singular value decomposition, there exist  matrices $U\in \RR^{d\times r}$ and $V\in \RR^{d\times r}$ with orthonormal columns such that
		\[X= U\Sigma V^*=\sum_{\ell=1}^r \alpha_\ell u_\ell v^*_\ell,\]
		where $\Sigma=\Diag_r(\alpha_\ell)\in \RR^{r\times r}$ and $u_\ell, v_\ell$ are columns of $U,V$.
		
		Define $\la A,B \ra:= \sum_{i,j}A_{ij}B_{ij}$ for $A,B\in\RR^{d\times d}$.
		Then making $V^{(\ell)}=u_\ell v^*_\ell\in \RR^{d\times d}$, applying above result, we have $$\la V^{(\ell)}, X\ra= \sum_{i=1}^r \alpha_i \la u_i v_i^*, u_\ell v_\ell^*  \ra = \sum_{i=1}^r \alpha_i \sum_{m}^{d} (u_i)_m (u_\ell)_m \sum_{n}^{d} (v_j)_n^* (v_\ell)_n^*  = \alpha_{\ell},$$ and thus by \propref{prop3}, there exist 2D convolutional layers $\{h_j\}_{j=1}^J$ such that $h_J(X)_\ell=\la V^{(\ell)}, X\ra=\alpha_\ell$.
	\end{proof}

	
	\begin{remark}
		\corref{cor4} sheds light on the potential of 2D CNNs and theoretical support for their success in image restoration tasks. For example, we can view 2D CNNs as a low rank approximation, which is of wide usage, such as in data compression, as they leverage the low rank assumption to capture essential structural information within a matrix/image denoted as $X$. The Eckart-Young-Mirsky theorem plays a fundamental role in this context, demonstrating that the best rank $k$ approximation to $X$ under Frobenius norm can be expressed as $X_k = \sum_{\ell=1}^k \alpha_\ell u_{\ell} v_{\ell}^*$ for $k \leq r$. When singular values of $X$ are acquired using convolutional layers, subsequent layers, including fully connected layers and other network architectures, can utilize these values or play a role in mapping singular values back to its low rank approximation for various image restoration tasks.
	\end{remark}
	
	\section{Approximation Analysis}
	
	In this section, applying the above analysis, we provide approximation estimates of functions which are essentially defined in a lower dimensional space. The approximation will be conducted by convolutional layers and one (or two) fully connected layers. To be simple, our focus is only on the case of convolutions of kernel size 2 and stride 2. Precisely, given input dimension $d$ and feature dimension $m$ with $d\gg m$, we denote by $\mathcal{H}(d,m, J)$ a set of convolutional layers $\{h_{j}\}_{j=1}^J$ as defined in the previous section with kernel size $2$, stride $2$ and the number of channels for $\{h_j\}_{j=1}^J$ being $2, 2^2, \ldots, 2^{k_0-1}, md/2^{k_0}, md/2^{k_0+1},\ldots , m$ respectively, where
	$$k_0=\mathop{\arg\min}_k\{md/2^k\leq 2^k\}=\lceil\log_4(md)\rceil.$$
Let $N(\mathcal{H}(d,m, J))$ be the total number of free parameters contained in $\mathcal{H}(d,m,J)$. Then
\begin{align*}
  N(\mathcal{H}(d,m, J))\leq& 2 [2+2^3+\cdots+2^{2k_0-3}]+2^{k_0}\f{md}{2^{k_0}}\\
  &+2(md)^2[\f{1}{2^{2k_0+1}}+2^{2k_0-1}+\f{1}{2^{2\log_2 d}}]\\
  \leq &\f 2 3(4^{k_0-1}-2)+md+\f 4 3\f{(md)^2}{4^{k_0}}\leq \f 4 3(4^{k_0}+\f{(md)^2}{4^{k_0}})+md\leq 8md,
\end{align*}
here in the last step we use the fact that $4^{k_0-1}\leq md\leq 4^{k_0}$. This bound increases linearly with respect to the dimension $d$, which is much smaller than the number of parameters in deep fully connected neural networks. In the following, we will show that with CNNs for feature extraction, a neural network for classification can approximate functions with smoothness or defined on manifolds well. Note that for any $\{v^{(\ell)}\}_{\ell=1}^m\subset \RR^d$, since
\[\dim \left(\bigcup_{\ell=1}^m\mathcal P(v^{(\ell)})\right)\leq \min\{2^k, md/2^k\}, \quad k=1,\ldots, \lceil\log_2 d\rceil,\]
there exists $\{h_j\}_{j=1}^J$ such that \eqref{inner-product} holds.

	Recall that the Sobolev space $H^q(\RR^m)$ on $\RR^m$ with $q\in\NN$ consists
	of all functions $f$ on $\RR^m$ such that all partial derivatives of $f$ up to order $q$ are square
	integrable on $\RR^m$ and $\|f\|_{H^{q}}=\sum_{|\alpha|_1\leq q}\|D^{\alpha} f\|_2$.
	
	The following lemma gives an approximation estimate of smooth functions by using shallow neural networks. It can be found from the proof of Theorem B in \cite{Hinton} and \cite{Barron}.
	\begin{lem}\label{Mao}
		Let $m \in \mathbb{N}$ and $f \in H^{q}\left(\mathbb{R}^{m}\right)$ with an integer index $q>\frac{m}{2}+2$. Then for every $N \in \mathbb{N}$, there exists a linear combination of ramp ridge functions of the form
		$$
		f_{N}(y)=\beta_{0}+\alpha_{0} \cdot y+\sum_{k=1}^{N} \beta_{k} \sigma\left(\alpha_{k} \cdot y-t_{k}\right)
		$$
		with $\beta_{k} \in \mathbb{R},\left\|\alpha_{k}\right\|_{1} \leq 1, t_{k} \in[0,1]$ such that
		$$
		\left\|f-f_{N}\right\|_{C\left([-1,1]^{m}\right)} \leq c_{0}\|f\|_{H^{q}\left(\mathbb{R}^{m}\right)} \sqrt{\log (N+1)} N^{-\frac{1}{2}-\frac{1}{m}}
		$$
		for some universal constant $c_{0}>0$.
	\end{lem}

	\begin{thm}\label{1D-approx}
		Let $J, m, d\geq 1$ be integers. 
		For any integer $n$ and a function $F:X\to \RR$ in the form of
		$F(x)=f(Vx)$ with $V$ being a matrix of size $m\times d$ and $f\in H^{q}\left(\mathbb{R}^{m}\right)$ with an integer index $q>\frac{m}{2}+2$,
		there exist convolutional layers $\{h_{j}\}_{j=1}^J\in \mathcal H(d,m, J)$ with $J=[\log_2 d]$ and a shallow neural network $\Psi: \RR^m\to \RR$ with width $n$ activated by ReLU such that
		$$\|\Phi(x)-F(x)\|_{C([-1,1]^m)}\leq C \|F\|_{H^{q}\left(\mathbb{R}^{m}\right)} n^{-1/2},$$
		where $\Phi(x)=\Psi(h_J(x))$.
	\end{thm}

	\begin{proof}

		For all features $\{v^{(\ell)}\}$ that are rows of $V$,
		applying Theorem~\ref{1D-extractation}, there exists a multi-channel DCNN $\{h_j\}_{j=1}^J$ such that
		\[h_J(x)_\ell=\langle x, v^{(\ell)}\rangle, \ \ell=1,\ldots,m,\]
		which implies that $h_J(x)$ coincide with $Vx$.
		Furthermore, by \lemref{Mao}, there is a fully connected ReLU layer $\Psi: \RR^{m}\to \RR$ with $n$ hidden neurons such that
		\[|\Psi(y)-f(y)|\leq  C \|f\|_{H^{q}\left(\mathbb{R}^{m}\right)} n^{-1/2}.\]
	\end{proof}

Our next approximation analysis is for functions in a high dimensional space which are essentially supported on a much lower-dimensional manifold.		
	Our next approximation analysis is for functions in a high dimensional space which are essentially supported on finitely many low-dimensional hyperplanes. It is worthwhile to note that given the number of hyperplanes $T$ and the largest dimension $m$, for $t=1,\ldots, T$,
	\[Z_t=\{\sum_{\ell=1}^m a^{(t)}_\ell v^{(t,\ell)}+u^{(t)}: (a^{(t)}_1,\ldots, a^{(t)}_m)\in [0,1]^m\}\]
	where $\{v^{(t,\ell)}\}_{\ell=1,\ldots,m}$ is orthonormal for any fixed $t=1,\ldots, T$, one can find $\tilde u^{(t)}\in \RR^d$ such that
	\[Z_t=\{\sum_{\ell=1}^m a^{(t)}_\ell v^{(t,\ell)}+\tilde u^{(t)}: (a^{(t)}_1,\ldots, a^{(t)}_m)\in [2t-1,2t]^m, t=1,\ldots, T \}.\]
	Additionally, we say $\{Z_t\}$ is well-separated with radius $\mu>0$ if there exist $z^{(t)}\in \RR^d$  such that
	\begin{equation}\label{well-separated}
	\|x-z^{(t)}\|^2-\|x-z^{(t_0)}\|^2\geq \mu, \quad \text{if}\ x\in Z_{t_0}, t\neq t_0.
	\end{equation}
	With the above notation, setting $\Omega_t=[2t-1,2t]^m$, $t=1,\ldots, T$, and $\Omega=[1,2T]^m$ we have the following result.
	
	\begin{figure}[http]
		\centering
		\subfigure[]{\includegraphics[width=8cm]{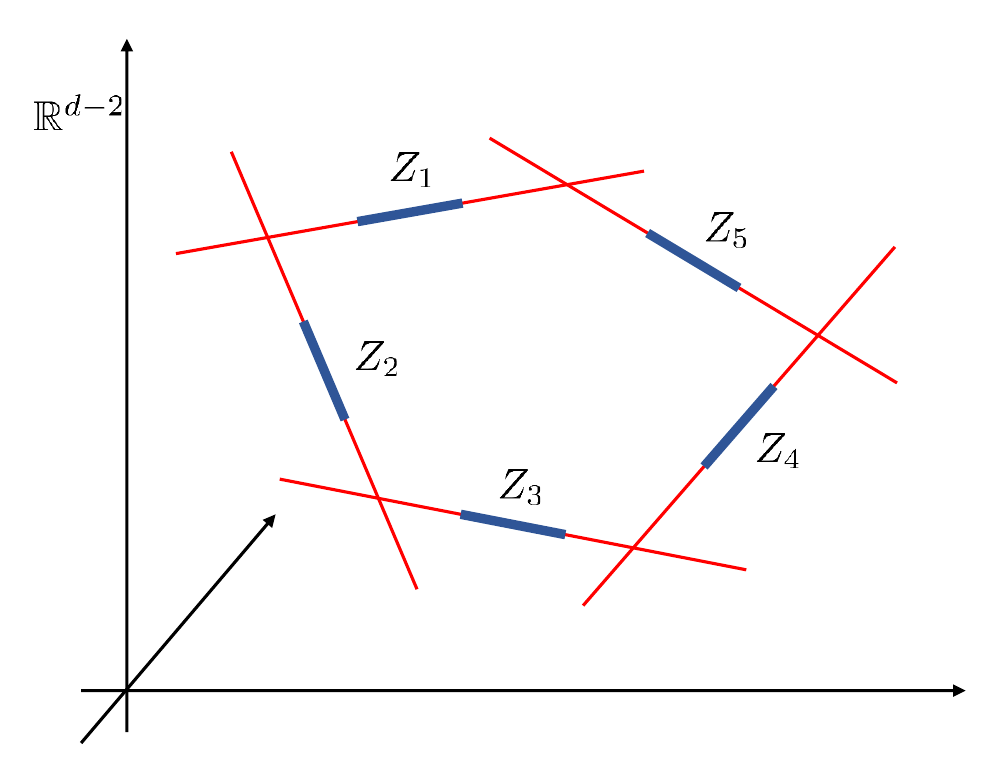}}
		\caption{Affine spaces.}
		\label{acts}
	\end{figure}
	
	\begin{thm}\label{thm:affine}
		Let $d, m, J\geq 1$ be integers, $X=\bigcup_{t=1}^T Z_t$, and $f\in H^{q}\left(\mathbb{R}^{m}\right)$ with an integer index $q>\frac{m}{2}+2$.
		Suppose $\{Z_t\}$ is well-separated with radius $\mu>0$ and a function $F:X\to \RR$ is defined by $F(x)=f(a^{(t)}_1,\ldots, a^{(t)}_m)$ for $x=\sum_{\ell=1}^m a^{(t)}_\ell v^{(t,\ell)}+u^{(t)}$. Then
		there exist convolutional layers $\{h_{j}\}_{j=1}^J\in \mathcal H(d,m, J)$ with $J=[\log_2 d]$ and a fully connected network $\Psi$ with residual connection and width $2T^2$, $mT$, $n$ activated by ReLU such that
		$$\|\Phi(x)-F(x)\|\leq C n^{-1/2}, \quad \forall x\in X,$$
		where $\Phi(x)=\Psi(h_J(x))$.
	\end{thm}
	
	\begin{proof}
		Let $B>0$ and $q_\ell^{(t)}\in \RR$ such that for all $x\in X$ , $t=1,\ldots, T$ and $\ell=1,\ldots, m$,
		$$|\la x-u^{(t)}, v^{(t,\ell)}\ra|\leq B$$
		and
		\[z^{(t)}=\sum_{\ell=1}^m q_\ell^{(t)}v^{(t,\ell)}+u^{(t)}, \quad t=1,\ldots, T.\]
		Let $x\in X$, say
		\[x=\sum_{\ell=1}^m a^{(t_0)}_\ell v^{(t_0,\ell)}+u^{(t_0)},\]
		for some $t_0=1,\ldots, T$. Then for $t=1,\ldots,T$,
		\[\|x-z^{(t)}\|^2=\|x\|^2+\|z^{(t)}\|^2-2\sum_{\ell=1}^m q^{(t)}_\ell\la x, v^{(t,\ell)}\ra-2\la x, u^{(t)}\ra,\]
		and if $t\neq t_0$,
		\[\|x-z^{(t)}\|^2-\|x-z^{(t_0)}\|^2\geq \mu.\]
		
		Having $y_{t,\ell}:=\la x, v^{(t,\ell)}\ra, y_{t,m+1}:=\la x, u^{(t)}\ra$,  we can construct a fully connected layer
		\begin{align*}
		h_{J+1}(y)_{t,t'}=&\s\left(\sum_{\ell=1}^{m+1} \f{2q^{(t)}_\ell y_{t,\ell}-2q^{(t')}_\ell y_{t',\ell}}{\mu}+\f{\|z^{(t')}\|^2-\|z^{(t)}\|^2
		}\mu\right)\\
		&-\s\left(\sum_{\ell=1}^{m+1} \f{2q^{(t)}_\ell y_{t,\ell}-2q^{(t')}_\ell y_{t',\ell}}{\mu}+\f{\|z^{(t')}\|^2-\|z^{(t)}\|^2
		}\mu-1\right)
		\end{align*}
		where $q^{(t)}_{m+1}=1$. Note that
		\[\sum_{t' = 1}^T h_{J+1}(y)_{t,t'}\left\{
		\begin{array}{ll}
		= T-1, & \hbox{$t=t_0$;} \\
		\leq T-2, & \hbox{$t\neq t_0$.}
		\end{array}
		\right.
		\]
		Then by setting
		\begin{align*}\
		h_{J+2}\left(h^{(J+1)}, h^{(J)}\right)_\ell=&\sum_{t=1}^{T} \s\left(y_{t,\ell}-B\left(T-1-\sum_{t'=1 }^Th_{J+1}(y)_{t,t'}\right)-\la u^{(t)}, v^{(t,\ell)}\ra\right)\\
		=&\la x-u^{(t_0)}, v^{(t_0,\ell)}\ra
		\end{align*}
		It is obvious that $h
		_{J+2}$ can be realized by a fully connection network with two hidden layers and width $2 T^2 $, $mT$. Following the proof of \thmref{1D-approx}, it is easy to obtain the result.
		%
		%
		%
		%
		%
		%
	\end{proof}

%
%

We are able to extend approximation results to data that are actually from a lower dimensional manifold utilizing a nearly metric-preserved linear transformation.

\begin{lem}\label{Manifold-2}[Theorem $3.1$ of \cite{Baraniuk}].
	Let $\mathcal{M}$ be a compact $d_{\mathcal{M}}$-dimensional Riemannian submanifold of $\mathbb{R}^{d}$ having condition number $1 / \tau$, volume $V$, and geodesic covering regularity $\mathcal{R}$. Fix $\delta \in(0,1)$ and $\gamma \in(0,1)$. Let $A=\sqrt{\frac{d}{\tilde m}} \Phi$, where $\Phi \in \mathbb{R}^{\tilde m \times d}$ is a random orthoprojector with
	$$
	\tilde m=\mathcal{O}\left(\frac{d_{\mathcal{M}} \ln \left(d V \mathcal{R}\tau^{-1} \delta^{-1}\right) \ln (1 / \gamma)}{\delta^{2}}\right) .
	$$
	If $\tilde m \leq d$, then with probability at least $1-\gamma$, the following statement holds: For every $x_{1}, x_{2} \in \mathcal{M}$
	$$
	(1-\delta)\left|x_{1}-x_{2}\right| \leq\left|A x_{1}-A x_{2}\right| \leq(1+\delta)\left|x_{1}-x_{2}\right| .
	$$
\end{lem}

For functions supported on a lower-dimensional manifold, we could show the following result (proof is postponed to Appendix B).
	\begin{thm}\label{Manifold-1}
		Let $\mathcal{M}$ be a compact $m$-dimensional Riemannian submanifold of $\RR^d$ that satisfies the conditions in \lemref{Manifold-2}. If $F$ is a Lipchitz-1 function on $\mathcal M$, for $n\in \NN$ and $1\leq p<\infty$,  there exists convolutional layers $\{h_{j}\}_{j=1}^J\in \mathcal{H}(d,\tilde m, J)$ with $J= [\log_2 d]$, $\tilde m = O(m\ln d)$
and a fully connected network $\Psi: \RR^{\tilde m}\to \RR$ with width $4\tilde m \lceil \f{n}{2} \rceil$ and $\lceil \f{n}{2} \rceil^{\tilde m}$ activated by ReLU such that
		$$\|\Phi-F\|_{L^p(\mathcal M)}\leq C_2 n^{-1},$$
		where $\Phi(x)=\Psi(h_{J}(x))$, $C$ only depends on $\mathcal M$ and the Lipchitz constant of $f$. The total number of parameters is less than $4d + (6\tilde m+1)2^{1-\tilde m}n^{\tilde m}$.
	\end{thm}
	
\section{Discussion and Conclusion}
Many existing studies have primarily concentrated on 1D convolutions with a stride of $1$. Bao et al. \cite{bao2014approximation} concentrates on cyclic convolution networks for compositional structures. In the work of Zhou \cite{Zhou2020_1}, the approximation error of functions from Sobolev spaces using 1D CNNs is characterized by $(1/J)^{\f{1}{2}+ \f{1}{d}}$ where $J$ represents the depth. Subsequently, Zhou \cite{Zhou2020_2} introduced downsampling operators in combination with 1D CNNs, establishing approximation results for Lipschitz-$\alpha$ ridge functions. With downsampling, CNNs are also proved to be able to produce any given fully connected neural network (FNN). Later on, multichannel 1D CNNs are explored in \cite{liu2021besov} to approximate Besov functions over lower dimensional manifolds, while their approach builds upon a relationship between CNNs and FNNs, as developed in \cite{oono2019approximation}. The work \cite{he2022approximation} of multichannel 2D CNNs follows a similar approach with the key observation that larger kernels can be reduced to $3\times 3$ kernels and approximation rates of a subspace of $L_2$ space are established.

Our construction offers a significant advantage: it achieves a linear transformation with a depth of only $O(\ln d)$, a substantial reduction when compared to the $O(d)$ depth required in prior work \cite{oono2019approximation}. Consequently, when applying Theorem 1 and Theorem 2 to existing results for FNNs approximating functions from Sobolev space, H\"{o}lder space, or Lipschitz space, there is a notable improvement in depth.
Moreover, these advantages readily extend to 2D Convolutional Neural Networks (CNNs), as demonstrated in Proposition 3. This adaptability empowers us to construct structures (as outlined in Theorem 7, Theorem 8, and Theorem 9) that closely resemble those used in practical applications. However, those stride $1$ convolutions are difficult to be extended from 1D to 2D inputs. The discussion of our 2D CNNs can be directly extended to $d $ dimensional inputs.
The last advantage is that our results require the condition ``stride = kernel size''. With this property, the dimension of outputs of hidden layers is continuously reducing, which is more suitable for applications since usually the memory is limited and downsampling is utilized to design networks frequently.

The above advantages bring benefits to approximating functions. For instance, when inputs are intrinsically defined on lower-dimensional spaces, CNNs can effectively serve as dimension-reduction operators (since CNNs are good at extracting textures of images), while FNNs can take on roles as predictors or classifiers. Theorem 7 yields approximation results with a fixed depth $J = O(\ln d)$ and then the approximation error being solely tied to the FNN component. Additionally, it reveals that if the input dimensionality is essentially $m$ rather than $d$, the total number of parameters of convolutional layers will not exceed $O(md)$. This feature facilitates scenarios where data from different datasets, albeit possessing similar lower-dimensional structures or serving different tasks (e.g., classification or super-resolution), can retain the CNN as a shared dimension-reduction operator, enabling transfer learning.
The implications of Theorem 8 are particularly relevant to classification problems, as it signifies the capability of neural networks to assign labels to inputs originating from diverse affine spaces while only utilizing the most representative members ($v^{t,\ell}$). Notably, in Theorem 9, for Lipschitz-1 functions with a tolerance of $\varepsilon$, the total parameter count remains bounded by $\varepsilon^{-m\ln d}$, a substantial reduction compared to the lower bound of $\varepsilon^{-d/r}$ for Sobolev functions with regularity $r$.

In conclusion, we established approximation analysis for deep convolutional layers. It explicitly reveals the role of multi-channels in extracting features of signals and reduce dimension of data with low-dimensional models.
%
%
%
%
%
%
%
One the other hand, it also verifies the practical observation that convolutions increase the receptive field of a feature from the $\ell$th layer to the $(\ell+1)$th layer. In other words, each feature in the next layer captures a larger spatial
region in the input layer. For example, when using a $2\times 2$ filter convolution successively
in three layers, the activations in the first, second, and third hidden layers capture pixel
regions of size $2\times 2$, $4\times 4$ and $8\times 8$, respectively, in the original input image.
This is a natural consequence of the fact that features in later layers capture complex characteristics of the image over larger spatial regions, and then combine the simpler features in earlier layers. With slightly modification, the same analysis can work as well for signals with more than one channels and convolutions in higher dimensions, like 3D convolution. One can just convert the problem to the 1D convolution case through a proper ordering operator for tensor 3 matrices.

\section{Acknowledgement}
The authors are supported partially by InnoHK initiative, the Government of the HKSAR, and Laboratory for AI-Powered Financial Technologies,
the
Research Grants Council of Hong Kong [Projects No. C1013-21GF, No. 1308020, No.11315522 and  No. 11308121], the Germany/Hong
Kong Joint Research Scheme [Project No. G-CityU101/20].

\section{Appendix A}

\begin{proof}[Proof of  \thmref{1D-extractation}]
	We give a constructive proof to verify the existence. Since one can always use large bias $B^{(j)}_{\ell}$ to push the convolution part in each convolutional layer above zero, without loss of generality we assume that at each layer convolution parts are positive and thus ReLU activation is omitted throughout the proof.
	
	Let $r_{k,j}\in \RR^{2^k}$ and $n_k\in\NN$ such that
	$$\S\left(\bigcup_{\ell=1}^m \P_{2^k}( v^{(\ell)})\right)=\{r_{k,j}, j=1,\ldots, n_k\}, $$
	for all $k=1,\ldots, J$. For convenience, we view $r_{k,j}$ as row vectors without the transpose. Note that $\P_{2^J}(v)=\{v\}$ for any $v\in \RR^{2^J}$ and $\S\left(\bigcup_{\ell=1}^m \P_{2^J}(v^{(\ell)})\right)=\{v^{(\ell)}\}$. Particularly, let $r_{J,\ell}=v^{(\ell)}$ for $\ell=1,\ldots, m$.

	Now fixing any $k=1,\ldots, J-1$ and $j=1,\ldots, n_{k+1}$,
	let $u^{(t)}_{k+1,j}\in \RR^{2^{k}}$, $t=1,2$ such that
	$$(u^{(1)}_{k+1,j},u^{(2)}_{k+1,j})=  r_{k+1,j}\in \RR^{2^{k+1}}.$$
	Then since $r_{k+1,j}\in \Span\{\bigcup_{\ell=1}^m\P_{2^{k+1}}(v^{(\ell)})\}$, there are $c_n \in \RR$ and $b_n=(b^{(1)}_n, b^{(2)}_n)\in \bigcup_{\ell=1}^m\P_{2^{k+1}}(v^{(\ell)})$ such that
	\[r_{k+1,j} =\sum_{n} c_n (b^{(1)}_n, b^{(2)}_n), \quad \text{i.e.} \ u^{(t)}_{k+1,j}=\sum_{n} c_n  b^{(t)}_n,\quad t=1,2, \]
	which implies that
	$u^{(t)}_{k+1,j}\in \Span\{\bigcup_{\ell=1}^m\P_{2^{k}}(v^{(\ell)})\}$ and furthermore, $u^{(t)}_{k+1,j}=\sum_{i=1}^{n_k} w^{(k+1,t)}_{i,j} r_{k,i}$ for $t=1,2$ and some $w^{(k+1,t)}_{i,j}\in \RR$.
	On the other hand,  it yields that
	\begin{align*}
	r_{k+1,j}=&(u^{(1)}_{k+1,j}, 0_{2^k})+(0_{2^k},u^{(2)}_{k+1,j})\\
	=&\sum_{i=1}^{n_k}(w^{(k+1,1)}_{i,j}, 0)\cdot \left(
	\begin{array}{cc}
	r_{k,i} & 0_{2^k} \\
	0_{2^k} &0_{2^k}\\
	\end{array}
	\right)+\sum_{i=1}^{n_k}(0,w^{(k+1,2)}_{i,j})\cdot \left(
	\begin{array}{cc}
	0_{2^k} & 0_{2^k} \\
	0_{2^k} &r_{k,i}\\
	\end{array}
	\right)\\
	=& \sum_{i=1}^{n_k}(w^{(k+1,1)}_{i,j}, w^{(k+1,2)}_{i,j})\cdot \left(
	\begin{array}{cc}
	r_{k,i} & 0_{2^k} \\
	0_{2^k} & r_{k,i}\\
	\end{array}
	\right),
	\end{align*}
where $0_{2^k}$ is the zero row vector of length $2^k$.
	Now taking diagonal operator for both sides,
	\begin{align}\label{keyeqn2}
	&\Diag_{2^{J-k-1}}(r_{k+1,j})\notag\\
	=&\sum_{i=1}^{n_k} \Diag_{2^{J-k-1}}\left(w_{i,j}^{(k+1,1)}, w_{i,j}^{(k+1,2)}\right) \Diag_{2^{J-k-1}}
	\left(
	\begin{array}{cc}
	r_{k,i} & 0_{2^k} \\
	0_{2^k} & r_{k,i} \\
	\end{array}
	\right)\notag\\
	=&\sum_{i=1}^{n_k} \Diag_{2^{J-k-1}}\left(w_{i,j}^{(k+1,1)}, w_{i,j}^{(k+1,2)}\right) \Diag_{2^{J-k}}(r_{k,i}).
	\end{align}
	Set filters of size two to be  
	\[
	W_{i,j}^{(k)}=\left\{
	\begin{array}{ll}
	r_{1,j}, & \hbox{$k=1,i=1, j=1,\ldots, n_1$ ;} \\
	(w^{(k,1)}_{i,j}, w^{(k,2)}_{i,j}), & \hbox{$k=2,\ldots, J$, $i=1,\ldots, n_{k-1}$ and $j=1,\ldots, n_k$,}
	\end{array}
	\right.
	\]
and define $h^{(k)}: \RR^{2^J}\to \RR^{2^{J-k}\times n_{k}}$ by
$$h_k(X)_{j}=\sum_{i=1}^{n_{k-1}}W^{(k)}_{i,j} *_{2} h_{k-1}(X)_i ,\quad j=1,\dots,n_k.$$
	Realizing that for any $x\in \RR^{2k}$, $\Diag_k(a,b)x=(a,b)*_2 x$, then \eqref{keyeqn2} iteratively yields that
	$\la \bar v^{(\ell)}, x\ra=h_J(x)_\ell$, where $h_J$ is as defined by Definition 1 with kernels $\{W_{i,j}^{(k)}\}$ and stride 2.

\end{proof}

\section{Appendix B}

Let $\omega_f$ denotes the modulus of continuity of a continuous function $f$ on $[0, 1]^m$.
To prove Theorem~\ref{Manifold-1}, we need the following lemma.

	\begin{lem}\label{contn}
		Suppose $f$ to be a continuous function on $[0,1]^m$. For any $n\in\NN$ and $\d>0$, there exist $A_\d\subset [0,1]^m$ with $|A_\delta|\leq \d$ and a ReLU network  $\phi$ with two hidden fully connected layers of width $4m \lceil \f{n}{2} \rceil$ and $\lceil \f{n}{2} \rceil^m$ such that
		\[|f(x)-\phi(x)|\leq C\omega_f(n^{-1}), \quad \forall\ x\in [0,1]^d\setminus A_\d.\]
		Additionally, by taking $\d$ sufficiently small, for $1\leq p<\infty$,
		\[\|f-\phi\|_{L_p([0,1]^m)}\leq  C\omega_f(n^{-1}).\]
	\end{lem}
	
	\begin{proof}
		For $n\in \NN$ and $0<\delta < 1$, we define
		\[\psi_\delta(t) := \f n{\delta}\left(\sigma(t + 1/n) - \sigma(t+ 1/n-\delta/n ) -\sigma(t-1/n+ \delta/n) + \sigma(t-1/n)\right).\]
		One can see that
		\begin{align*}
		\psi_\delta(t)=\left\{
		\begin{array}{ll}
		0, & \hbox{if $x\in (-\infty, -1/n]\cup [1/n, \infty)$,} \\
		1, & \hbox{if $t \in [-1/n+\d/n,1/n-\d/n]$,} \\
		\text{linear}, & \hbox{otherwise}
		\end{array}
		\right.
		\end{align*}
		Then for a given $f \in C\left([0,1]^d \right)$, we define the following neural network $\phi$
		\begin{align*}
		\phi( x) = \sum_{ k\in\Lambda} f(  k/n ) \Psi_\d( x;  k),
		\end{align*}
		where $\Psi_\d( x;  k)=\sigma\left[\sum_{i=1}^{d} \psi_\delta({ e}_i\cdot  x-k_i/n )-d+1\right]$, $\Lambda=\{1,3,\ldots, 2[n/2]-1\}^d$ and $[m]$ is the smallest integer  $\geq m$.
		For $ k\in \Lambda$, let $B_{ k,\d}=[k_1-1/n+\d/n,k_1 + 1/n-\d/n]\times \cdots \times[k_d-1/n+\d/n,k_d + 1/n-\d/n]$ be the cube with center $ k/n$. Then for any $ x \in \bigcup_{ k\in \Lambda} B_{ k,\d} $, there is a unique $ k^*\in \Lambda$ such that
		$$| e_i\cdot  x -  k_i^*/n|\leq  (1-\d)/n, \quad \forall \ i=1,\ldots, d.$$ and hence $\Psi_\d( x;  k^*) = 1$ and  $\Psi_\d( x;  k) =0$ for $ k\neq  k^*$. It yields that for any $ x\in \bigcup_{ k\in \Lambda} B_{ k,\d}$,
		\begin{align*}
		&	|f( x) - \phi( x)|\leq \sum_{ k\in \Lambda}\Psi_\d( x;  k)| f( x)-f( k/n)| \\
		\leq& | f( x)-f( k^*/n)|\leq \omega_f(\sqrt{d}/n).
		\end{align*}
		On the other hand, let $A_\d=[0,1]^d \setminus \bigcup_{ k\in \Lambda} B_{ k}$. Then $|A_\d|\leq d \delta $.
		
		Finally, note that $\phi$ can be realized by a ReLU network with two hidden layers of width $4m \lceil \f{n}{2} \rceil$ and $\lceil \f{n}{2} \rceil^m$.
		
		%
		%
	\end{proof}
Now we are in the position to prove Theorem~\ref{Manifold-1}.
	\begin{proof}[Proof of Theorem~\ref{Manifold-1}]
		From the above, taking $\delta=\gamma=1/2$, there exists a matrix $A$ in size $\tilde m \times d$ such that
		$$
		\f 1 2\left|x_{1}-x_{2}\right| \leq\left|A x_{1}-A x_{2}\right| \leq \f 3 2\left|x_{1}-x_{2}\right|,
		$$
		where $\tilde m=C_{\tau, V, R} m\ln d$ for some constant $C_{\tau, V, R}$. For any $\var \in(0,1)$, let $\mathcal M_\var=\bigcup_{x\in \mathcal M} \{z\in \RR^d: \|x-z\|\leq \var\}$. Following the methodology of \cite[Sect 4.2]{Shen}, one can construct a lower dimensional function $f$ such that
		\[|F(x)-f(Ax)|\leq C\var \]
		for all $x\in \mathcal M_\var$ and $|f(y_1)-f(y_2)|\leq C \|y_1-y_2\|$ for all $y_1,y_2\in A\mathcal M_\var$. Let $v^{(\ell)}$ be the $\ell$th collumn of $A$. Then by Theorem~\ref{1D-extractation}, there exists $\{h_j\}_{j=1}^J\in \mathcal H(d,\tilde m, J)$ with $J\leq \log_2 d$ such that $v^{(\ell)}\cdot x=h_J(x)_\ell$ for all $\ell=1,\ldots, \tilde m$. Combining with Lemma~\ref{contn}, we complete the proof.
	\end{proof}

\bibliographystyle{abbrvnat}

\end{document}